\documentclass[11pt]{article}
\usepackage[sf]{titlesec}
\usepackage{graphicx,wrapfig}
\usepackage{setspace}
\usepackage{booktabs}
\usepackage{amsmath}
\usepackage{comment}

\usepackage[ruled,vlined]{algorithm2e}

\usepackage{svg}
\usepackage{amsthm}
\usepackage{mathtools}

\usepackage[utf8]{inputenc}
\usepackage{tikz}

\usetikzlibrary{
   arrows.meta, 
   bending
}

\usepackage{hyperref}
\usepackage{scalerel,amssymb}
\usetikzlibrary{shapes,arrows}
\usepackage{graphics} 
\usepackage{epsfig} 
\usepackage{subcaption}
\usepackage[square,numbers]{natbib}
\usepackage{authblk}
\usepackage[margin=0.6in]{geometry}

\usepackage{float}

\newtheorem{prop}{Proposition}

\restylefloat{figure}

\providecommand*{\ped}[1]{\ensuremath{_\mathrm{#1}}}
\providecommand*{\ap}[1]{\ensuremath{^\mathrm{#1}}}

\usepackage{enumitem}
\setlist[itemize]{noitemsep}

\titlespacing\section{0pt}{12pt plus 2pt minus 2pt}{0pt plus 2pt minus 2pt}
\titlespacing\subsection{0pt}{12pt plus 2pt minus 2pt}{0pt plus 2pt minus 2pt}
\titlespacing\subsubsection{0pt}{12pt plus 2pt minus 2pt}{0pt plus 2pt minus 2pt}

\usepackage{scalefnt}
\usepackage{lipsum}

\providecommand*{\diff}%
   {\@ifnextchar^{\DIfF}{\DIfF^{}}}
\def\DIfF^#1{%
   \mathop{\mathrm{\mathstrut d}}%
      \nolimits^{#1}\gobblespace
}
\def\gobblespace{%
   \futurelet\diffarg\opspace}
\def\opspace{%
   \let\DiffSpace\!%
   \ifx\diffarg(%
      \let\DiffSpace\relax
   \else
      \ifx\diffarg\[%
         \let\DiffSpace\relax
      \else
         \ifx\diffarg\{%
            \let\DiffSpace\relax
         \fi\fi\fi\DiffSpace}
\makeatother

\title{Slot-hopping Enabled Loiter Guidance and Automation for Fixed-wing UAV Corridors\footnote{Copyright © 2026 by Pradeep J, Kedarisetty Siddhardha, and Ashwini Ratnoo. Published by the American Institute of Aeronautics and Astronautics, Inc., with permission.} \footnote{Presented in AIAA SCITECH 2026 Forum, Publised version DoI: https://doi.org/10.2514/6.2026-2368}}

\author{Pradeep J\footnote{B.Tech (Honours) Student, Department of Mechatronics Engineering;pradeepj0406@gmail.com} and Kedarisetty Siddhardha \footnote{Assistant Professor, Department of Mechatronics Engineering; siddhardhak@iitbhilai.ac.in}}
\affil{Indian Institute of Technology Bhilai, Bhilai, 491002, India}
\author{Ashwini Ratnoo\footnote{Professor, Department of Aerospace Engineering; ratnoo@iisc.ac.in. Associate Fellow AIAA.}}
\affil{Indian Institute of Science, Bengaluru, 560012, India}

\date{\vspace{-7ex}}

\begin{document}

\maketitle

\begin{abstract}
This paper addresses the problem of traffic congestion management in fixed-wing unmanned aerial vehicle (UAV) corridors by further developing a recently introduced loiter-lane framework.
A semi-cooperative guidance strategy is developed for inserting fixed-wing UAVs into a loiter lane with minimal disruption to the UAVs already operating within it, while enabling a more compact fixed-wing UAV corridor.
Building on the concepts of cooperative and non-disruptive loiter-lane insertion, the proposed strategy makes the incoming UAV first attempt, within its speed bounds, to rendezvous with an existing empty loiter slot. 
If direct insertion is infeasible, a minimal number of loitering UAVs perform coordinated slot hopping to create a suitably positioned empty slot.
The feasibility and performance of the method are demonstrated through numerical simulations.
\end{abstract}

\section{Introduction}
The rapid advancement in UAVs in the aspects of novel configurations~\cite{ozdougan2022design, siddhardha2018novela, siddhardha2018novelb}, battery and actuator technologies~\cite{xiao2023design, polonsky2021rotor}, and advanced guidance algorithms has expanded their applications across military~\cite{siddhardha2023intercepting}, research~\cite{siddhardha2019autonomous}, civilian~\cite{ghamari2022unmanned}, and commercial sectors~\cite{betti2024uav}.
With this increasing demand, it has become imperative to integrate UAV operations into Class-G airspace.
Achieving this integration requires establishing structured UAV traffic management (UTM) frameworks and defining airspace rules that account for UAVs of different capabilities, such as vertical take-off and landing (VTOL) and fixed-wing UAVs.

Significant progress has been made toward UTM through initiatives such as NASA UTM~\cite{kopardekar2014unmanned}, SESAR U-space~\cite{huttunen2019u}, CORRIDRONE~\cite{tony2020corridrone},and EURODRONE~\cite{lappas2020eurodrone}.
Additional research has addressed enabling technologies, including geofencing~\cite{bhise2022signed}, safe flight corridors~\cite{lozano2025risk}, base-station placement~\cite{lyu2016placement}, urban airspace design~\cite{yang2024review}, lane-changing strategies~\cite{nagrare2022multi}, and collision-avoidance mechanisms~\cite{aldao2022lidar}.
A core element of UTM is corridor planning, typically posed as a path-planning problem in which static airspace restrictions are treated as obstacles. 
However, most existing approaches target multirotor UAVs, whose hovering and low-speed agility simplify operations in constrained airspace. 
In contrast, fixed-wing UAVs, which are better suited for long-range missions, higher payloads, and energy-efficient flight, have received comparatively limited attention despite their operational importance.

A key challenge in fixed-wing UAV corridors is traffic congestion, as fixed-wing vehicles require a minimum forward speed to remain airborne. 
To address this, an earlier work~\cite{kedarisetty2023cooperative} introduced the concept of a loiter lane, a circular holding path connected to the main corridor through transit lanes, where UAVs can temporarily divert until main-lane congestion dissipates. 
In that study, a cooperative guidance law adjusted the speeds of all UAVs in the corridor to insert a target aircraft from the main lane into the loiter lane. 
Subsequently, the idea of equiangular loiter slots, virtual positions distributed along the loiter circle was proposed in~\cite{kedarisetty2025loiter}. 
That work developed a guaranteed insertion strategy in which an incoming UAV adjusts its speed to occupy any empty slot without disturbing loitering UAVs, under certain constraints on the main-lane-to-loiter-lane separation.
However, the works~\cite{kedarisetty2023cooperative, kedarisetty2025loiter} do not address the problem of UAV insertion considering a reduced loiter separation distance, which is essential for minimizing the width of the air corridor.

In this paper, we extend the concepts by developing a semi-cooperative loiter guidance strategy that enables reliable insertion even when the distance between the main lane and the loiter lane is reduced.
The proposed method allows loitering UAVs to vacate or ``hop” slots only when necessary. 
If an incoming UAV flying within its speed bounds can reach an available slot, no loitering UAVs change their loitering slots. 
Otherwise, a minimal number of loitering UAVs perform coordinated slot hopping to create a suitably positioned empty slot for the incoming aircraft. 
This semi-cooperative approach provides a flexible, scalable framework for fixed-wing corridor management, allowing efficient congestion resolution while minimizing disruption in the loiter lane.

The rest of the paper is arranged as follows: Section~\ref{sec:2} describes the scenario and formulates the problem.
Section~\ref{sec:3} presents the corridor design, which includes determining the loiter radius and its separation from the main lane. 
Section~\ref{sec:4} details the proposed guidance and automation algorithm for conflict-free UAV insertion into the loiter lane.
Simulation results validating the approach are provided in Section~\ref{sec:5}. Section~\ref{sec:6} concludes the paper.

\section{Scenario \& Problem Statement\label{sec:2}}

The fixed-wing UAV corridor considered in this work is illustrated in Fig.~\ref{fig:FWUAV_corridor}. 
The corridor consists of the main lane, the loiter circle, transit lanes, and transit links.
The main lane and loiter circle are connected through incoming and outgoing transit lanes, which are tangentially attached to the loiter circle to enable smooth entry. 
The entry and exit between the transit and main lanes are achieved through transit links, which are circular arcs that comply with the minimum turn-radius constraints of fixed-wing UAVs. 

\begin{figure}[h]
  \centering
  \includegraphics[width=9.5cm]{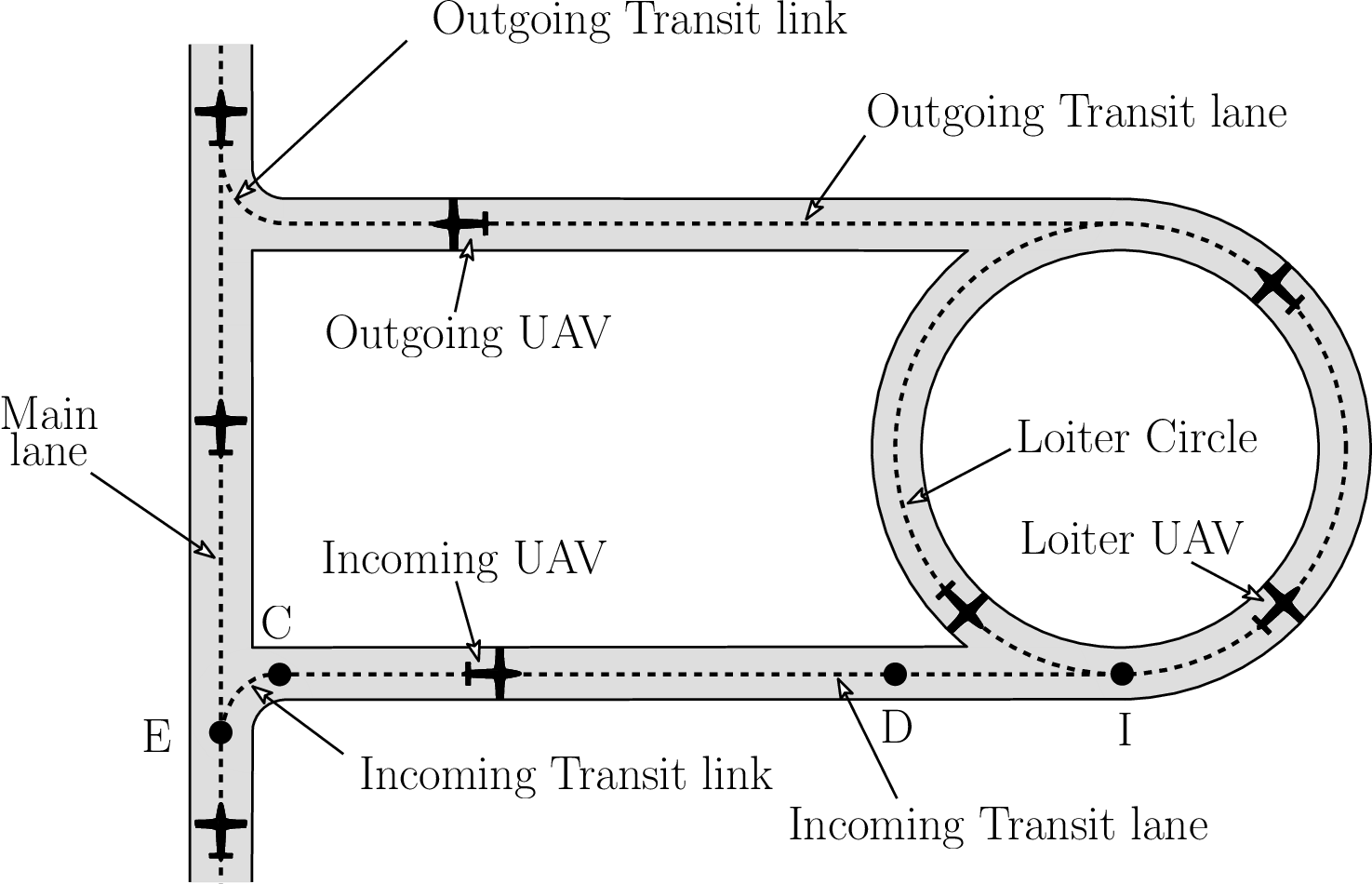}
  \caption{Fixed-wing UAV corridor}
  \label{fig:FWUAV_corridor}
\end{figure}

\begin{figure}[h]
    \centering
    \includegraphics[width=7.5cm]{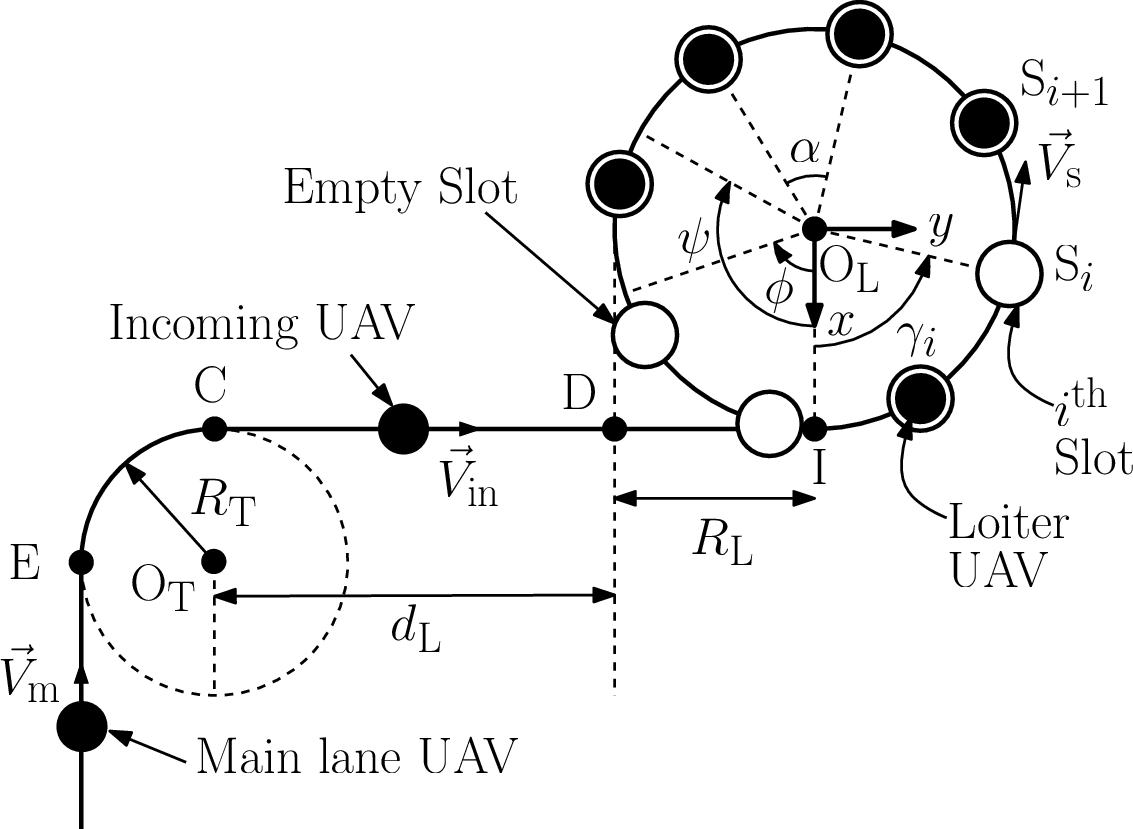}
    \caption{Corridor Geometry}
    \label{fig:corr_geo}
\end{figure}

The incoming UAV in the main lane to be inserted into the loiter lane exits at point E and enters the incoming transit link EC, as shown in Fig.~\ref{fig:FWUAV_corridor}. 
Upon reaching point C, the UAV enters the transit lane (CI). 
Throughout the transit link and transit lane, the UAV maintains a constant speed of $V\ped{in}$.
This speed is determined such that the UAV avoids collision with other loiter UAVs while entering the loiter circle at point I.
After entering the loiter circle, the UAV loiters with the speed $V\ped{s}$, which is the speed maintained by all the loiter UAVs under steady conditions.
Throughout the corridor, including the main lane, transit lanes, transit links, and loiter circle, all fixed-wing UAVs operate within the maximum and minimum speed bounds of $V\ped{min}$ and $V\ped{max}$, respectively.

The lane geometry of the proposed corridor is illustrated in Fig.~\ref{fig:corr_geo}.
The parameters $R\ped{T}$, $d\ped{L}$, and $R\ped{L}$ denote the transit link radius, separation distance between the main lane and the loiter circle, and the loiter circle radius, respectively.
The centers of the transit link arc and the loiter circle are represented by $\mathrm{O}\ped{T}$ and $\mathrm{O}\ped{L}$.
The loiter circle contains $N$ rotating equiangularly spaced slots $\mathrm{S}_i$, which are virtual positions that the incoming UAVs can occupy. 
These slots revolve with a constant speed of $V\ped{s}$.
The angular separation between consecutive slots is denoted by $\alpha$.

\textbf{Problem Statement:}  For a given number of virtual slots $N$, UAV speed bounds $\left[ V\ped{min},\,V\ped{max}\right]$, and minimum safety distance $d\ped{s}$, determine a) $R\ped{L}$ and the feasible range of $d\ped{L}$ such the an incoming UAV can occupy a virtual slot $\mathrm{S}_i$ with $\gamma_i \in \left( 0,\,2\pi \right]$, $i\in \left[1,\,N \right]$, and b) determine the loiter UAVs' and incoming UAV speeds  $V_{\mathrm{L}_i}$ and $V\ped{in}$, respectively, for a conflict-free insertion.

\section{Design of Fixed-wing UAV Corridor\label{sec:3}}

The radius $R_\mathrm{L}$ is chosen to ensure a minimum safe separation distance \( d_\mathrm{s} \) between any two adjacent slots, as well as between the incoming UAV and all slots in the loiter lane (except the slot to be occupied).  

\begin{prop}
    For $V\ped{in} \geq V\ped{s}$, the choice of the loiter circle radius
    \begin{equation}
    R_\mathrm{L}= \frac{d_\mathrm{s}}{2\mathrm{sin}^2\left(\dfrac{\pi}{N}\right)}  , \quad \text{for } N>1 \; \mathrm{and} \; V\ped{in} \geq V\ped{s} 
    \label{eq:req_RL}
\end{equation} 
accommodates N ($\geq 2$) equiangular virtual slots with a minimum safe distance $d\ped{s}$ between all UAVs.
\end{prop}
\begin{proof}
    The proof for this proposition is provided in detail in our previous work \cite{kedarisetty2025loiter}.
\end{proof}

The design of the corridor involves determining the loiter circle separation from the main lane such that there exists a $V\ped{in} \in \left[ V\ped{min},\,V\ped{max} \right]$ for rendezvous with any empty loiter slot.

\begin{prop}
For a given $R_\mathrm{L}$, $N$, $V\ped{min}$, $V\ped{max}$, $R_\mathrm{T}$,
the minimum required value of $d_\mathrm{L}$ is given by
\begin{equation}
d\ped{L_{min}} = 
\begin{dcases}
    \begin{matrix}
        \left(\dfrac{2\pi R_\mathrm{L}}{N}\right) \left(\dfrac{k_\mathrm{v}}{k_\mathrm{v}-1}\right) - \dfrac{\pi R_\mathrm{T}}{2} -R_\mathrm{L}, \quad  \mathrm{if}\; & \left(\dfrac{2\pi R_\mathrm{L}}{N}\right) \left(\dfrac{k_\mathrm{v}}{k_\mathrm{v}-1}\right) \geq  \dfrac{\pi R_\mathrm{T}}{2} + R_\mathrm{L} \\
         &   \\
        0, \quad & \mathrm{otherwise}
    \end{matrix}
\end{dcases}
\label{eq:min_dl}
\end{equation}
where $k\ped{v} = \dfrac{V\ped{max}}{V\ped{min}}$
\end{prop}

\begin{proof} 
The minimum angular displacement required for the slot from the insertion point $I$ for successful insertion when the incoming UAV approaches with velocity $V_{\max}$ is denoted by $\phi$, as shown in Fig.~\ref{fig:corr_geo}. 
Similarly, the maximum angular displacement that the slot can traverse from the insertion point $I$ when the incoming UAV approaches with velocity $V_{\min}$ is denoted by $\psi$.

\medskip

The necessary condition for the existence of at least one slot in the loiter circle at any given time is
\begin{equation}
\frac{2\pi}{N} \leq \psi - \phi
\label{eq:condition}
\end{equation}
where
\[
\phi = \left(d_\mathrm{L} + R_\mathrm{L} + \frac{\pi R_\mathrm{T}}{2}\right)\frac{1}{R_\mathrm{L} k_\mathrm{v}},
\qquad 
\psi = \frac{1}{R_\mathrm{L}}\left(d_\mathrm{L} + R_\mathrm{L} + \frac{\pi R_\mathrm{T}}{2}\right).
\]

\medskip

By substituting $\phi$ and $\psi$ into \eqref{eq:condition} we obtain
\[
    \frac{2\pi}{N} \leq \frac{1}{R_\mathrm{L}}\left (d_\mathrm{L}+R_\mathrm{L}+ \frac{\pi R_\mathrm{T}}{2}\right) \left({1-\frac{1}{k_\mathrm{v}}}\right)
\]

From the inequality by rearranging the terms ,we can find minimum $d_\mathrm{L}$
\[
d_\mathrm{L}  \geq \left(\frac{2\pi R_\mathrm{L}}{N}\right) \left(\frac{k_\mathrm{v}}{k_\mathrm{v}-1}\right) - \frac{\pi R_\mathrm{T}}{2} -R_\mathrm{L}
\]

The length of $d_\mathrm{L}$ is always non-negative, i.e., $d_\mathrm{L} \ge 0$.
\end{proof}

\section{Guidance \& Automation Algorithm\label{sec:4}}

The guidance algorithm is designed to accommodate an incoming UAV in some reachable slot, regardless of whether that slot is currently occupied. 
If the required slot is already occupied, the UAV currently in that slot, along with any other adjacent loiter UAVs, will move to the nearest empty slot, thereby creating a vacancy for the incoming UAV.
Towards this, the first step is to identify all the feasible slots (occupied or unoccupied) that the incoming UAV can rendezvous with at point I.
For the incoming UAV in the main lane to be inserted into the loiter circle, it will travel along the transit link and transit lane (ECI), covering a distance $D_{\mathrm{L}}$ before reaching the loiter circle. 
\begin{equation}
    D_{\mathrm{L}}=\frac{\pi R_{\mathrm{T}}}{2} + d_{\mathrm{L}} +R_{\mathrm{L}}
\end{equation}
The time taken for the incoming UAV to reach the point of insertion I is
\begin{equation}
    t\ped{in} = \dfrac{D\ped{L}}{V\ped{in}}
\end{equation}

\noindent  The maximum and minimum time required for the incoming UAV to travel from point B to point I are
\begin{equation}
    t\ped{min} = \dfrac{D\ped{L}}{V\ped{max}},\; \; t\ped{max} = \dfrac{D\ped{L}}{V\ped{min}}
\end{equation}
The time needed by $i\ap{th}$ slot to reach the insertion point I is 
\begin{equation}
    t_i =\frac{(2 \pi - \gamma_i)R_\mathrm{L}}{V_\mathrm{s}} 
    \label{eq:delta_a}
\end{equation}
The necessary condition for successful incoming UAV insertion is for some $\mathrm{S}_i$, $i\in \left[1,\,N \right]$ to satisfy the following condition
\begin{equation}
    t\ped{min} \leq t_i \leq t\ped{max}
\end{equation}
All virtual slots that satisfy this feasibility condition constitute the set $\mathrm{S}\ped{f}$, and the corresponding set of feasible insertion times is denoted by $t\ped{f}$.
The set $\mathrm{S}\ped{f}$ is the union of two disjoint subsets $\mathrm{S}\ped{uf}$ and $\mathrm{S}\ped{of}$, which represent the feasible unoccupied and feasible occupied slots, respectively.
The corresponding times at which these slots reach point I are denoted by $t\ped{uf}$ and $t\ped{of}$. The set of empty slots in the loiter circle is defined as $\mathrm{S}\ped{e}$ and the corresponding set of times at which these slots reach point I are  $t\ped{e}$.
If the set $\mathrm{S}\ped{uf}$ is non-empty, the first slot in this set is selected as the desired slot for insertion.
However, if this set is empty, the nearest empty slot in the loiter circle relative to the feasible occupied slots is chosen so that the minimum number of loiter UAVs need to hop their positions.
The algorithm for determining the desired slot for insertion, along with the loiter UAVs that must change their positions, is described in Algorithm~\ref{alg:del_t}.

Upon determining the desired slot for insertion, the incoming UAV speed is determined as 
\begin{equation}
    V\ped{in} = \dfrac{D\ped{L}}{t\ped{in}}, \quad t\ped{in} = \dfrac{\left(2\pi - \gamma\ped{d} \right) R\ped{L}}{V\ped{s}}
\end{equation}
where $\gamma\ped{d}$ is the angular position of the desired slot.
The selected loiter UAVs will hop one slot ahead to free a slot with maximum speed and empty out the desired slot for insertion. 
The time required for this slot to shift is
\begin{equation}
    t\ped{hop} = \dfrac{2\pi R\ped{L}}{N \left( V\ped{max} -V\ped{s} \right)}
\end{equation}
For successful insertion $t_\mathrm{hop} < t_\mathrm{min}$. This condition is satisfied by taking the minimum $d_\mathrm{L}$ from Eq.~\eqref{eq:min_dl}. 
The guidance algorithm that enables the switching on and off of the loiter slot hopping, along with determining the incoming UAV speeds, is provided in Algorithm ~\ref{alg:hopping}. 

\begin{algorithm}[H]
\caption{Algorithm to determine the intersection time of incoming UAV and the slots needs to be hopped}
\label{alg:del_t}
\SetKwInOut{KwIn}{Input}
\SetKwInOut{KwOut}{Output}
\SetKwInOut{Procedure}{Procedure}
\SetKwInOut{Initialize}{Initialize}

\KwIn{$N$,\,$t\ped{uf}$,\,$t\ped{of}$,\,${t}\ped{e}$,\,$\mathrm{S}\ped{e}$}
\KwOut{hopSlots,$t\ped{in}$} 
\Initialize{$t\ped{in}={none}$, hopSlots$=[]$}
\Procedure{}
\For{$j=1$ \textbf{to} $\mathrm{length}$($t\ped{uf}$)}{
    \If{($t\ped{uf}(j)>t\ped{min}$ $\&\&$ $t\ped{uf}(j)<t\ped{max}$)}
    {
     $t\ped{in}= t\ped{uf}(j)$ \;
     hopSlots$=[]$\;
     \textbf{return} $t\ped{in}$, hopSlots\;
    }
}
\If{$t_{in}=$none}
{   
        \For{$k=1$ \textbf{to} $N-1$}{

        \For{$j=1$ \textbf{to} $\mathrm{length}(\mathrm{t}\ped{e})$}{
            \If{$(t_{a-k}>t_\mathrm{min} \, \, \&\& \, \, t_{a-k}<t_\mathrm{max})$} 
            {
             $t_\mathrm{in}= t_{a-k}$\; 
             hopSlots=[$\mathrm{S}\ped{e}(j)-1$, $\mathrm{S}\ped{e}(j)-2$,...,$\mathrm{S}\ped{e}(j)-k$ ]\; 
             \textbf{return} $t_\mathrm{in}$, hopSlots\; 
            }
        }
    }
}
\Return{$t_{in}$, hopSlots} \\
\end{algorithm}

\begin{algorithm}[H]
\caption{Guidance algorithm for calculating UAV speeds}
\label{alg:hopping}
\SetKwInOut{KwIn}{Input}
\SetKwInOut{KwOut}{Output}
\SetKwInOut{Procedure}{Procedure}
\SetKwInOut{Initialize}{Initialize}

\KwIn{$N$, hopSlots,\, $t\ped{in}$,\, $t\ped{hop}$,\, $D\ped{L}$,\, $t\ped{count}$}
\KwOut{$V\ped{in}$,\,$V_{\mathrm{L}_i}$} 

\Procedure{}
 $V\ped{in}=\frac{D\ped{L}}{t\ped{in}}$ \;
\For{$i = 1$ \KwTo $N$}{
    \If{$\mathrm{isEmpty}$ {hopSlots}}{
     $V_{{L}_i}=V\ped{min}$\;
    }
    \Else{
        \If{$i \ \mathrm{in} \ \mathrm{hopSlots} \ \&\& \ (t\ped{count} \leq t\ped{hop}) $} {
         $V_{\mathrm{L}_i} = V\ped{max}$\; } 
        \Else{
         $V_{\mathrm{L}_i} = V\ped{min}$\;
        }
    }
}
\Return{$V_{\mathrm{L}_i}$ and $V\ped{in}$}\\
\end{algorithm}

\section{Simulation results\label{sec:5}}

The proposed loiter lane design and guidance algorithm are validated through simulations using the following UAV non-linear kinematic model
\begin{equation}
\label{eq:uav_model}
\begin{aligned}
\dot{x}& = V \cos\theta, \;
\dot{y} = V \sin\theta, \;
\dot{\theta} = \frac{a}{V}
\end{aligned}
\end{equation}
Here, $(x, y)$ denote the position of the UAV, and $\theta$ represents the flight path angle measured from the x-axis. 
$V$ and $a$ represent the velocity and lateral acceleration of the UAVs as provided by the guidance algorithm. 
We conduct two simulation cases: a) The set $\mathrm{S}\ped{uf}$ is non-empty (presence of an unoccupied slot in the feasible loiter slots) and b) The set $\mathrm{S}\ped{uf}$ is empty.
The parameters chosen for the simulation are listed in Table~\ref{table:params}.

\begin{table}[h]
\centering
\caption{Simulation Parameters}
\begin{tabular}{c c c c c c c c c}
\hline
\hline
Case & $V_{\min}$ & $V_{\max}$ & $R\ped{L}$ & $R\ped{T}$ & $d\ped{L}$ & $N$ & $d\ped{s}$ \\
\hline
1 & 15.0\,m/s & 35.0\,m/s & 200.0\,m & 80.0\,m  & 350.0\,m & 8 & 58.5\,m\\
2 & 15.0\,m/s & 35.0\,m/s & 100.0\,m & 80.0\,m  & 300.0\,m & 6 & 50.0\,m\\
\hline
\hline
\end{tabular}
\label{table:params}
\end{table}

\begin{figure}[H]
    \centering
    \includegraphics[width=17.5cm]{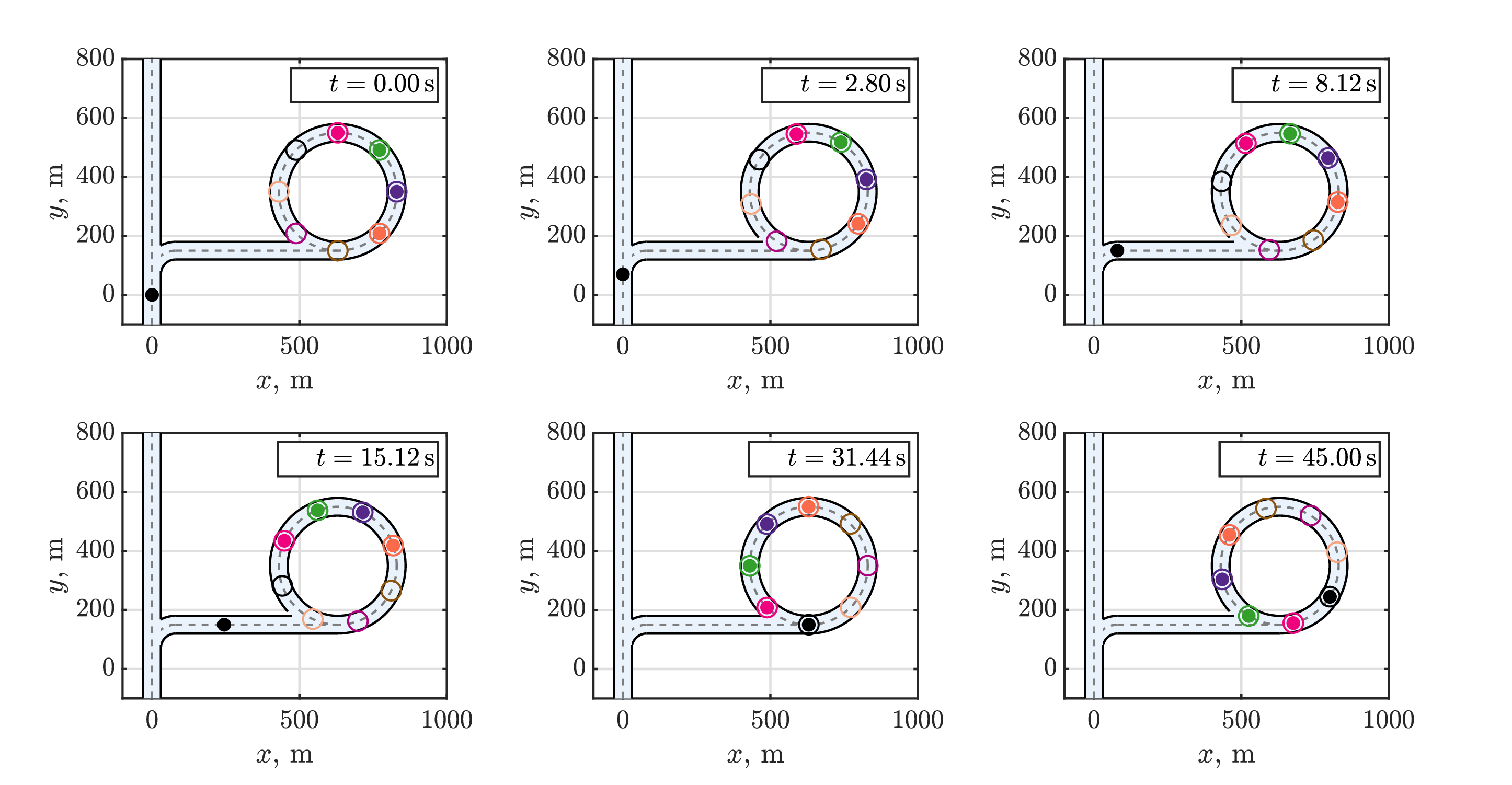}
    \caption{Case-1 Trajectory Plots}
    \label{fig:sim1a}
\end{figure}

Figure~\ref{fig:sim1a} illustrates the first simulation scenario with the sequence of events where a single incoming UAV (shown in black) is inserted into the loiter circle. 
At \(t = 0\,\mathrm{s}\), the incoming UAV passes along the main lane at \((0,0)\). At \(t = 2.8\,\mathrm{s}\), it enters the transit path. 
At this moment, the desired velocity of the incoming UAV is computed, and an appropriate empty slot is selected. 
By \(t = 31.44\,\mathrm{s}\), the UAV reaches the chosen slot at the insertion point, and continues to move along the occupied loiter slot as shown at \(t = 45\,\mathrm{s}\).
In this scenario, there are no slot position changes in the loiter circle as there exists an empty feasible loiter slot for the incoming UAV to occupy.
The velocity, lateral acceleration, flight angle, and 
$d_\mathrm{sep}$  profiles of the incoming UAV are presented in Fig.~\ref{fig:sim1b}. Here, $d_\mathrm{sep}$ denotes the distance between the incoming UAV and the desired slot.
This distance is evaluated only after the UAV exits the main lane, since the slot is assigned while entering the transit link.
\\

\begin{figure}[H]
    \centering
    \includegraphics[width=17cm]{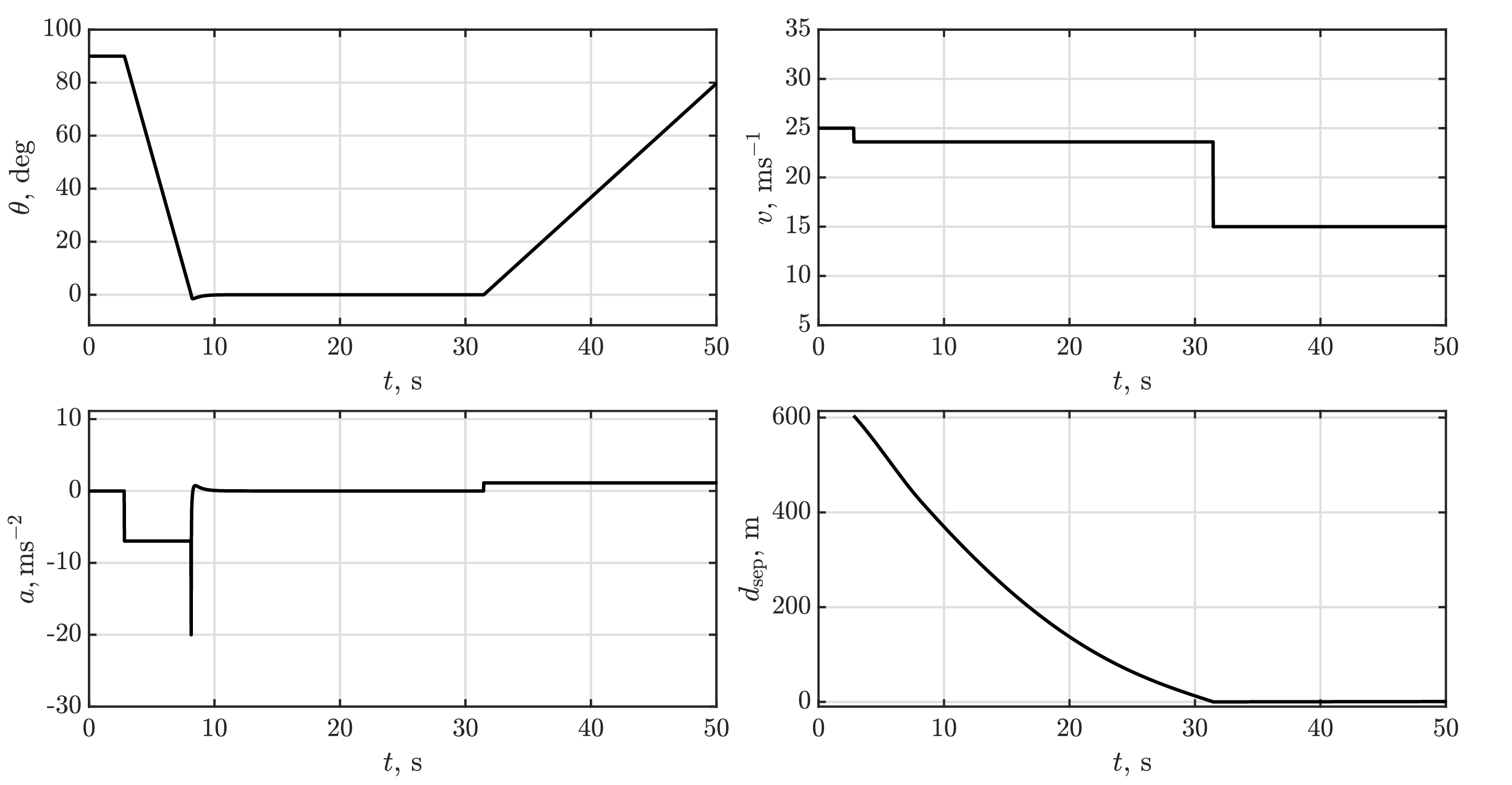}
    \caption{Case-1 Analysis plot}
    \label{fig:sim1b}
\end{figure}

\begin{figure}[H]
    \centering
    \includegraphics[width=17.5cm]{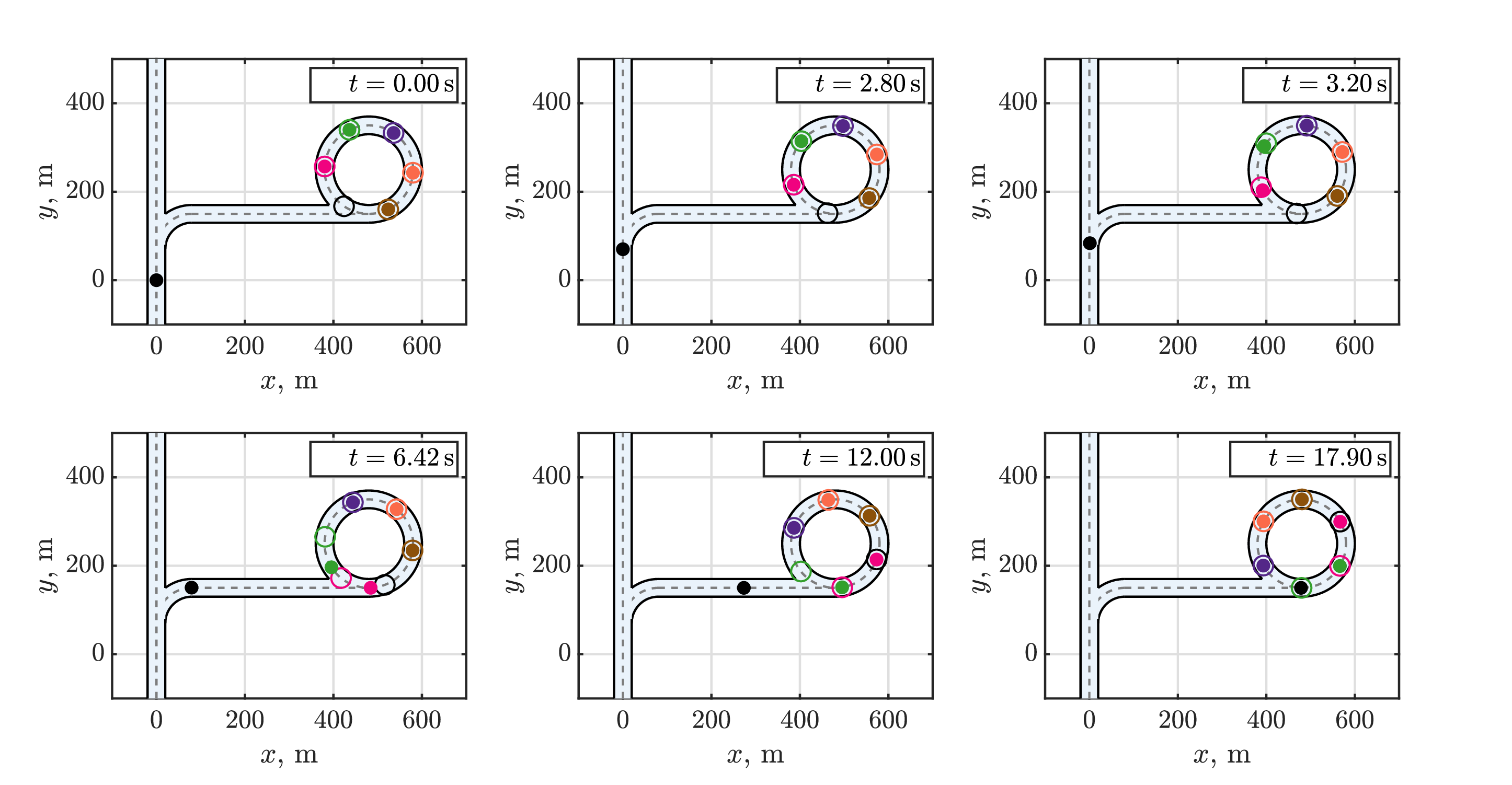}
    \caption{Case-2 Trajectory Plots}
    \label{fig:sim2a}
\end{figure}
 
The second simulation presents the scenario where the incoming UAV gets inserted into the loiter lane with the cooperation of the loiter UAVs, as shown in Fig.~\ref{fig:sim2a}. 
At \(t = 0.0\,\mathrm{s}\), the incoming UAV passes along the main lane lane at (0,0). 
By \(t = 2.8\,\mathrm{s}\), it enters the transit lane, the slot selection algorithm and begins the search for a feasible slot. 
Finding no immediate availability of feasible slot, it identifies an occupied yet feasible one and initiates the guidance algorithm to empty that slot for incoming UAV insertion.
The algorithm chooses the slot such that the minimum number of loiter UAVs shift. 
At \(t = 3.2\,\mathrm{s}\), the UAVs occupying the feasible slots already started to leave their prior slot positions at their maximum allowable velocities. 
The slot shifting is completed before the incoming UAV reaches I, as can be seen at \(t = 6.42\,\mathrm{s}\), and by \(t = 17.9\,\mathrm{s}\), the incoming UAV successfully enters the loiter lane. 
This scenario's analysis plots are presented in Fig.~\ref{fig:sim2b}.

\begin{figure}[H]
    \centering
    \includegraphics[width=17.0cm]{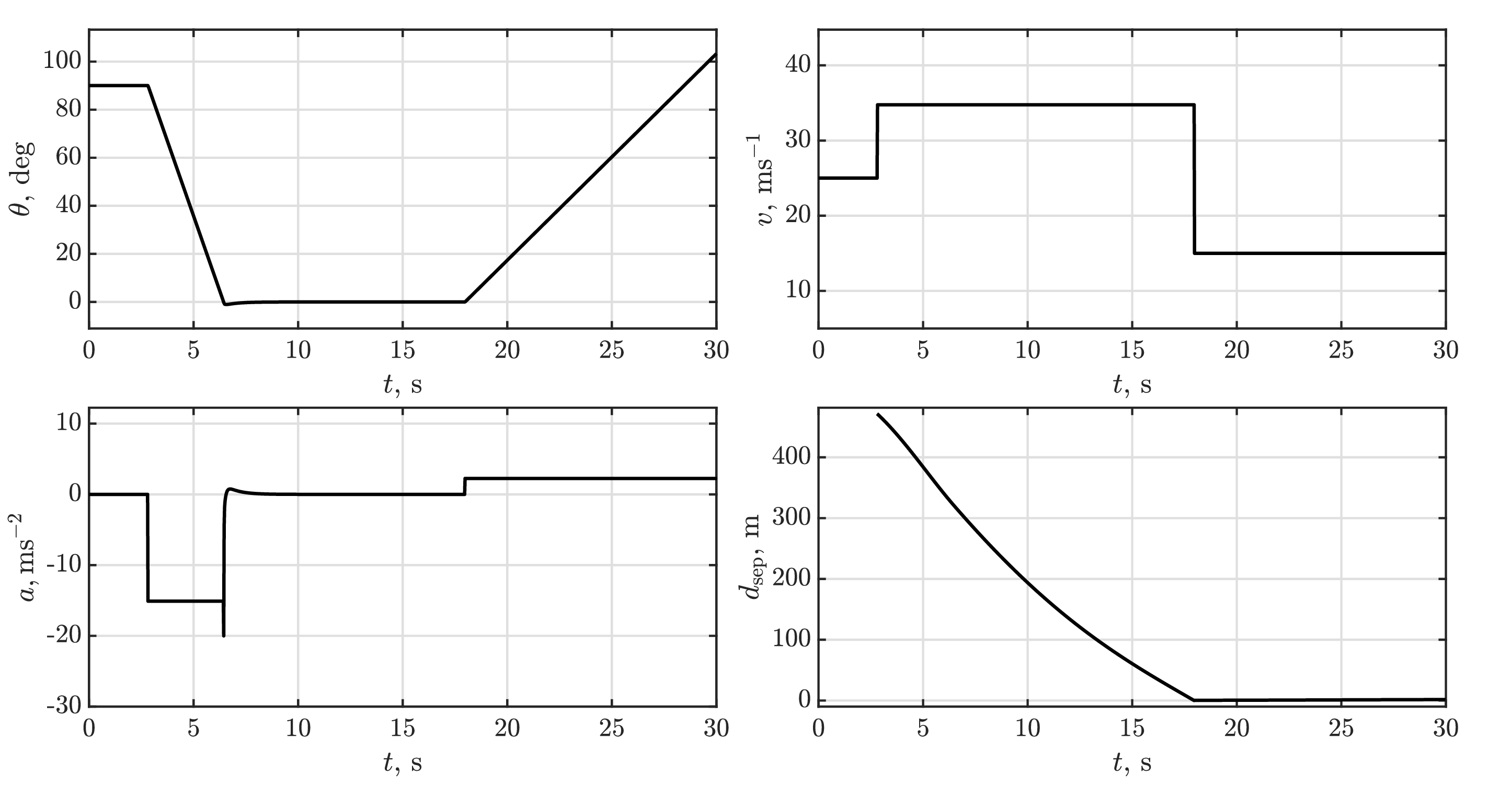}
    \caption{Case-2 Analysis Plots}
    \label{fig:sim2b}
\end{figure}

\section{Conclusion\label{sec:6}}
This paper introduces a semi-cooperative guidance algorithm for UAV insertion into the loiter circle, reducing the transit lane separation distance. 
The developed guidance algorithm safely diverts UAVs from the main lane and inserts them into the loiter lane without conflicts. 
The algorithm determines commands for both the incoming UAV and the loitering UAVs. The effectiveness of the guidance strategy is validated through simulation results.

\bibliography{sample}

\end{document}